\documentclass{article}

\usepackage{multicol}
\usepackage{microtype}
\usepackage{graphicx}
\usepackage{subfigure}
\usepackage{booktabs} 
\usepackage{amssymb}
\usepackage{bbm}
\usepackage[fleqn]{amsmath}
\usepackage{amsthm}
\usepackage{enumerate}
\usepackage{color}

\DeclareMathOperator{\sgn}{sgn}

\usepackage{hyperref}

\usepackage[accepted]{icml2018}

\icmltitlerunning{Curriculum Learning by Transfer Learning}

\newtheorem{theorem}{Theorem}
\newtheorem{lemma}{Lemma}

\newtheorem{corollary}{Corollary}

\newtheorem{defn}{Definition}

\newcommand{\bx}{{\mathbf x}}
\newcommand{\bX}{{\mathbf X}}
\newcommand{\bz}{{\mathbf z}}
\newcommand{\bw}{{\mathbf w}}
\newcommand{\ba}{{\mathbf a}}

\newcommand{\bs}{{\mathbf s}}

\newcommand{\iR}{\mathbb{R}}
\newcommand{\cR}{\mathcal{R}}
\newcommand{\cH}{\mathcal{H}}
\newcommand{\iX}{\mathbb{X}}
\newcommand{\iE}{\mathbb{E}}

\newcommand{\iO}{\vec{\mathcal{O}}}
\newcommand{\OO}{{\mathcal{O}}}

\begin{document}

\twocolumn[
\icmltitle{Curriculum Learning by Transfer Learning: Theory and Experiments with Deep Networks}

\begin{icmlauthorlist}
\icmlauthor{Daphna Weinshall}{huji}
\icmlauthor{Gad Cohen}{huji}
\icmlauthor{Dan Amir}{huji}
\end{icmlauthorlist}

\icmlaffiliation{huji}{School of Computer Science and Engineering, The Hebrew University of Jerusalem, Jerusalem 91904, Israel}

\icmlcorrespondingauthor{Daphna Weinshall}{daphna@mail.huji.ac.il}

\icmlkeywords{Curriculum learning, transfer learning, deep networks}

\vskip 0.3in
]






\printAffiliationsAndNotice{}  

\begin{abstract}

We provide theoretical investigation of curriculum learning in the context of stochastic gradient descent when optimizing the convex linear regression loss. We prove that the rate of convergence of an ideal curriculum learning method is monotonically increasing with the difficulty of the examples. Moreover, among all equally difficult points, convergence is faster when using points which incur higher loss with respect to the current hypothesis. We then analyze curriculum learning in the context of training a CNN. We describe a method which infers the curriculum by way of transfer learning from another network, pre-trained on a different task. While this approach can only approximate the ideal curriculum, we observe empirically similar behavior to the one predicted by the theory, namely, a significant boost in convergence speed at the beginning of training. When the task is made more difficult, improvement in generalization performance is also observed. Finally, curriculum learning exhibits robustness against unfavorable conditions such as excessive regularization. 


\end{abstract}

\section{Introduction}
\label{sec:intro}

Biological organisms can learn to perform tasks (and often do) by observing a sequence of labeled events, just like supervised machine learning. But unlike machine learning, in human learning supervision is often accompanied by a curriculum. Thus the order of presented examples is rarely random when a human teacher teaches another human. Likewise, the task may be divided by the teacher into smaller sub-tasks, a process sometimes called shaping \cite{krueger2009flexible} and typically studied in the context of reinforcement learning \citep[e.g. ][]{graves2017automated}. Although it remained for the most part in the fringes of machine learning research, curriculum learning has been identified as a key challenge for machine learning throughout \citep[e.g., ][]{mitchell1980need,mitchell2006discipline,wang2015basic}.

We focus here on curriculum learning based on ranking (or weighting as in \cite{bengio2009curriculum}) of the training examples, which is used to guide the order of presentation of examples to the learner. Risking over simplification, the idea is to first present the learner primarily with examples of higher weight or rank, later to be followed by examples with lower weight or rank.  Ranking may be based on the difficulty of each training example as evaluated by the teacher, from easiest to the most difficult. 

In Section~\ref{sec:theory} we investigate this strict definition of curriculum learning theoretically, in the context of stochastic gradient descent used to optimize the convex linear regression loss function. We first define the (ideal) difficulty of a training point as its loss with respect to the optimal classifier. We then prove that curriculum learning, when given the ranking of training points by their difficulty thus defined, is expected (probabilistically) to significantly speed up learning especially at the beginning of training. This theoretical result is supported by empirical evidence obtained in the deep learning scenario of curriculum learning described in Section~\ref{sec:deep-learning}, where similar behavior is observed. We also show that when the difficulty of the sampled training points is fixed, convergence is faster when sampling points that incur higher loss with respect to the current hypothesis as suggested in \cite{ShrivastavaGG16}. This result is \emph{not} always true when the difficulty of the sampled training points is not fixed.

But such ideal ranking is rarely available. In fact, the need for such supervision has rendered curriculum learning less useful in machine learning, since ranking by difficulty is hard to obtain. Moreover, even when it is provided by a human teacher, it may not reflect the true difficulty as it affects the machine learner. For example, in visual object recognition it has been demonstrated that what makes an image difficult to a neural network classifier may not always match whatever makes it difficult to a human observer, an observation that has been taken advantage of in the recent work on adversarial examples \cite{szegedy2013intriguing}. Possibly, this is one of the reasons why curriculum learning is rarely used in practice \citep[but see, e.g.,][]{zaremba2014learning,amodei2016deep,jesson2017cased}.  

In the second part of this paper we focus on this question - how to rank (or weight) the training examples without the aid of a human teacher. This is paramount when a human teacher cannot provide a reliable difficulty score for the task at hand, or when obtaining such a score by human teachers is too costly. This question is also closely related to transfer learning: here we investigate the use of another classifier to provide the ranking of the training examples by their presumed difficulty. This form of transfer should not be confused with the notion of transfer discussed in \cite{bengio2009curriculum} in the context of multi-task and life-long learning \cite{thrun2012learning}, where knowledge is transferred from earlier tasks (e.g. the discrimination of easy examples) to later tasks (e.g. the discrimination of difficult examples). Rather, we investigate the transfer of knowledge from one classifier to another, as in \emph{teacher classifier} to \emph{student classifier}. In this form curriculum learning has not been studied in the context of deep learning, and hardly ever in the context of 
other classification paradigms. 

\comment{
There are a number of issues to be addressed in order to design a method for transfer curriculum learning. First, what determines the difficulty of an example? Second, which classifier should be used to compute the curriculum? Should it be a more powerful network pre-trained to do a different task \citep[e.g. ][]{szegedygoing}, or a smaller network that can be quickly trained on the final task? Should we use the same network to continuously rank the training set in order to achieve an adaptive curriculum \citep[as in \emph{self-paced learning:} ][]{kumar2010self,jiang2015self}? In Section~\ref{sec:deep-learning} we investigate the first two possibilities, leaving the incorporation of self-paced learning into the scheme for later work. We focus primarily on transfer from a big network pre-trained on a different task.} Differently from previous work, it is not the instance representation which is being transferred but rather the ranking of training examples. 
Why is this a good idea? This kind of transfer assumes that a powerful pre-trained network is only available at train time, and cannot be used at test time even for the computation of a test point's representation. This may be the case, for example, when the powerful network is too big to run on the target device. One can no longer expect to have access to the transferred representation at test time, while ranking can be used at train time in order to improve the learning of the target smaller network (see related discussion of network compression in \cite{chen2015compressing,kim2015compression}, for example). 

In Section~\ref{sec:deep-learning} we describe our method, an algorithm which uses the ranking to construct a schedule for the order of presentation of training examples. In subsequent empirical evaluations we compare the performance of the method when using a curriculum which is based on different scheduling options, including 2 control conditions where difficult examples are presented first or when using arbitrary scheduling. The main results of this empirical study can be summarized as follows: (i) Learning rate is always faster with curriculum learning, especially at the beginning of training. (ii) Final generalization is sometimes improved with curriculum learning, especially when the conditions for learning are hard: the task is difficult, the network is small, or 
when strong regularization is enforced. These results are consistent with prior art \citep[see e.g. ][]{bengio2009curriculum}.

\section{Theoretical analysis}
\label{sec:theory}

We start with some notations in Section~\ref{sec:notations}, followed in Sections~\ref{sec:regression} by the rigorous analysis of curriculum learning when used to optimize the linear regression loss. In Section~\ref{sec:theory-simulations} we report supporting empirical evidence for the main theoretical results, obtained using the deep learning setup described later in Section~\ref{sec:deep-learning}.

\subsection{Notations and definitions}
\label{sec:notations}

Let $\iX=\{(\bx_i,y_i)\}_{i=1}^n$  denote the training data, where $\bx_i\in\iR^d$ denotes the $i$-th data point and $y_i$ its corresponding label. In general, the goal is to find a hypothesis $\bar h(\bx)\in{\cal H}$ that minimizes the risk function (the expected loss). In order to minimize this objective, Stochastic Gradient Descent (SGD) is often used with various extensions and regularization. 

We start with two definitions:
\begin{defn}[Ideal Difficulty Score]
\label{def:diff_score}
The difficulty of point $\bx$ is measured by its minimal loss with respect to the set of optimal hypotheses $\{L(\bar h(\bx_i),y_i)\}$.
\end{defn}
\begin{defn}[Stochastic Curriculum Learning]
\label{def:SGL}
SCL is a variation on Stochastic Gradient Descent (SGD), where the learner is exposed to the data gradually based on the difficulty score of the training points. 
\end{defn}

In vanilla SGD training, at each iteration the learner is presented with a new datapoint (or mini-batch) sampled from the training data based on some probability function ${\cal D}(\iX)$. In SCL, the sampling is biased to favor easier examples at the beginning of the training. This bias is decreased following some scheduling procedure. At the end of training, points are sampled according to ${\cal D}(\iX)$ as in vanilla SGD. 

In practice, an SCL algorithm should solve two problems: (i) Score the training points by difficulty; in prior art this score was typically provided by the teacher in a supervised manner. (ii) Define the scheduling procedure. 

\subsection{The linear regression loss}
\label{sec:regression}

Here we analyze SCL when used to minimize the linear regression model. Specifically, we investigate the differential effect of a point's \emph{Difficulty Score} on convergence towards the global minimum of the expected least squares loss, when the family of hypotheses ${\cal H}$ includes the linear functions $h(\bx)=\ba^t\bx +b$ and $y\in\iR$.

The risk function of the regression model is the following
\begin{equation}
\begin{split}
\cR(\iX,\bw) &=\iE_{{\cal D}(\iX)}L(h_\bw(\bx),y) \\
L(h_\bw(\bx_i),y_i) &= (h(\bx_i)-y_i)^2 = (\ba^t\bx_i+b-y_i)^2 \\
&\triangleq (\bx_i^t\bw -y_i)^2  \triangleq L(\bX_i,\bw)
\end{split}
\label{eq:reg}
\end{equation}
In the last transition above, $\bw=${\tiny $\left [\begin{array}{c}\ba\\ b\end{array}\right ]$}$\in\iR^{d+1}$. With some abuse of notation, $\bx_i$ denotes the vector {\tiny $\left [\begin{array}{c}\bx_i\\ 1\end{array}\right ]$}. $\bX_i$ denotes the vector $[\bx_i,y_i]$, with \emph{Difficulty Score} $L(\bX_i,\bar\bw)$. 

In general the output hypothesis $h_\bw(\bx)=\bx_i^t\bw$ is determined by minimizing $\cR(\iX,\bw)$  with respect to $\bw$. The global minimum $\bar\bw$ of the empirical loss can be computed in closed form from the training data. However, gradient descent can be used to find $\bar\bw$ with guaranteed convergence, which is efficient when $n$ is very large.

Recall that SCL computes a sequence of estimators $\{\bw_t\}_{t=1}^T$ for the parameters of the optimal hypothesis $\bar\bw$. This is based on a sequence of training points $\{\bX_t=[\bx_t,y_t]\}_{t=1}^T$, sampled from the training data while favoring easy points at the beginning of training. Other than sampling probability, the update step at time $t$ follows SGD:
\begin{equation}
\label{eq:ranked-GD}
\bw_{t+1} = \bw_t - \eta \frac{\partial L(\bX_t,\bw)}{\partial \bw} \vert_{\bw=\bw_t}
\end{equation}

\subsubsection*{Convergence rate decreases with difficulty}

The main theorem in this sub-section states that the expected rate of convergence of gradient descent is monotonically \emph{decreasing} with the \emph{Difficulty Score} of the sample $\bX_t$. We prove it below for the gradient step as defined in (\ref{eq:ranked-GD}). If the size of the gradient step is fixed at $\eta$, a somewhat stronger theorem can be obtained where the constraint on the step size being small is not required.

We first derive the gradient step at time $t$:
\begin{equation}
\label{eq:gradient_t}
\bs = -\eta \frac{\partial L(\bX_i,\bw)}{\partial \bw} = -2\eta (\bx_i^t \bw - y_i)\bx_i 
\end{equation}

Let $\Omega_i$ denote the hyperplane on which this gradient vanishes $\frac{\partial L(\bX_i,\bw)}{\partial \bw}=0$. This hyperplane is defined by $\bx_i^t\bw = y_i$, namely, $\bx_i$ defines its normal direction. Thus (\ref{eq:gradient_t}) implies that the gradient step at time $t$ is perpendicular to $\Omega_i$\comment{ as illustrated in Fig.~\ref{fig:omega-plain}}. Let $\bar\bz$ denote the projection of $\bar\bw$ on $\Omega_i$. Let $\Psi^2=L(\bX_i,\bar\bw)$ denote the \emph{Difficulty Score} of $\bX_i$. 

\comment{\begin{figure}[th!]
	\centering
	\includegraphics[width=0.4\textwidth]{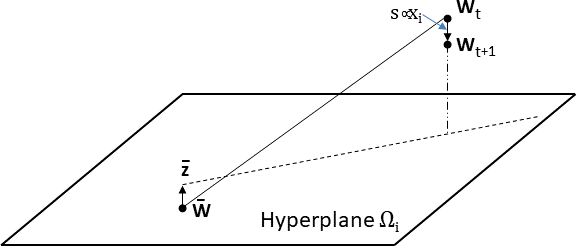}
    \caption{The geometry of the gradient step at time $t$. 
    \label{fig:omega-plain}}
\end{figure}}

\begin{lemma}
\label{lemma:1}
Fix the training point $\bX_i$. The \emph{Difficulty Score} of $\bX_i$ is
$\Psi^2= r^2 \Vert \bar\bw-\bar\bz \Vert^2$.
\end{lemma}

\begin{proof}
\begin{equation}
\begin{split}
\label{eq:diff-score}
\Psi^2 &=L(\bX_i,\bar\bw) = L(\bX_i,\bar\bz + (\bar\bw-\bar\bz))  \\
&= [ \bx_i^t\bar\bz + \bx_i^t (\bar\bw-\bar\bz) - y_i]^2 \\
&= [ \bx_i^t (\bar\bw-\bar\bz)]^2 = \Vert \bx_i \Vert^2 \Vert \bar\bw-\bar\bz \Vert^2 
\end{split}
\end{equation}
\comment{ 
The first transition in the last line follows from $\bar\bz\in\Omega_i \implies \bx_i^t\bar\bz- y_i=0$. The second transition follows from the fact that both $\bx_i$ and $(\bar\bw-\bar\bz)$ are perpendicular to $\Omega_i$, and therefore parallel to each other.}
\widowpenalty=10000
\end{proof}

Recall that $\bx_i,\bw\in\iR^{d+1}$. We continue the analysis in the parameter space $\bw\in\iR^{d+1}$, where parameter vector $\bw$ corresponds to a point, and data vector $\bx_i$ describes a hyperplane. In this space we represent each vector $\bx_i$ in a hyperspherical coordinate system $[r,\vartheta,\Phi]$, with pole (origin) fixed at $\bar\bw$ and polar axis (zenith direction) $\iO = \bar\bw-\bw_t$ (see Fig.~\ref{fig:obtuse}). $r$ denotes the vector's length, while $0\le\vartheta\le \pi$ denotes the polar angle with respect to $\iO$. Let $\Phi=[\varphi_1\ldots,\varphi_{d-1}]$ denote the remaining polar angles.

To illustrate, Fig.~\ref{fig:obtuse} shows a planar section of the parameter space, the $2D$ plane formed by the two intersecting lines $\iO$ and $\bar\bz-\bar\bw$. The gradient step $\bs$ points from $\bw_t$ towards $\Omega_i$. $\Omega_i$ is perpendicular to $\bx_i$, which is parallel to $\bar\bz-\bar\bw$ and to $\bs$, and therefore $\Omega_i$ is projected onto a line in this plane. We introduce the notation $\lambda=\Vert\bar\bw-\bw_t\Vert$. 

\begin{figure}[ht!]
	\centering
	\includegraphics[width=0.45\textwidth]{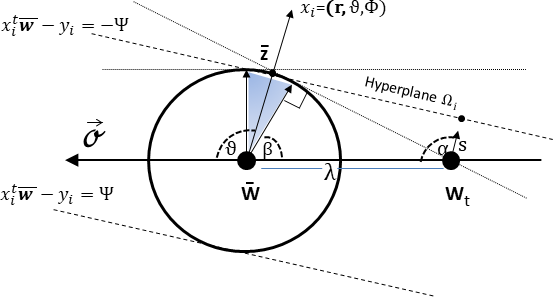}
    \caption{The $2D$ planar section defined by the vectors $\iO=\bar\bw-\bw_t$ and $\bar\bz-\bar\bw$. The circle centered on $\bar\bw$ has radius $\Vert\bar\bw-\bar\bz\Vert=\frac{{\Psi}}{\Vert\bx_i\Vert}$ from Lemma~\ref{lemma:1}. It traces the location of $\bar\bz$ for all the points $\bx_i$ with the same length $r$ and the same difficulty score $\Psi$. 
    \label{fig:obtuse}}
\end{figure}

Let $\bs_{\OO}$ denote the projection of the gradient vector $\bs$ on the polar axis $\iO$, and let $\bs_{\perp}$ denote the perpendicular component. From (\ref{eq:gradient_t}) and the definition of $\Psi$
\begin{equation}
\label{eq:gradient_full}
\begin{split}
&\bs = -2\eta\bx_i (\bx_i^t \bw_t - y_i) = -2\eta\bx_i[\bx_i^t (\bw_t-\bar\bw)\pm{\Psi}] \\
&\bs_{\OO}=\bs\cdot\frac{\bar\bw-\bw_t}{\lambda}=2\frac{\eta}{\lambda}[r^2\lambda^2\cos^2\vartheta\mp{\Psi}r\lambda\cos\vartheta]
\end{split}
\end{equation}

Let $\bx=(r,\vartheta,\Phi)$. Let $f_{{\cal D}(\iX)} = f(r,\vartheta,\Phi)f_Y(\vert y-\bx^t\bar\bw\vert)$ denote the density function of the data $\iX$. This choice assumes 
that the density of the label $y$ only depends on the absolute error $\vert y-\bx^t\bar\bw\vert$.

For the subsequent derivations we need the conditional distribution of the data $\iX$ given difficulty score $\Psi$. Fixing the difficulty score determines one of two labels $y(\bx,\Psi)=\bx^t\bar\bw\pm{\Psi}$. We further assume that both labels are equally likely\footnote{This assumption can be somewhat relaxed, but the strict form is used to simplify the exposition.}, and therefore $f_{{\cal D}(\iX)/_{\displaystyle{\Psi}}}([\bx,y]) = \frac{1}{2}f(r,\vartheta,\Phi)$. 

Let $\Delta(\Psi)$ denote the expected convergence rate at time $t$, given fixed difficulty score $\Psi$. 
\begin{equation}
\label{eq:delta}
\Delta(\Psi) = \iE[\Vert \bw_{t} - \bar\bw \Vert^2 - \Vert \bw_{t+1} - \bar\bw \Vert^2 /_{\displaystyle{\Psi}}]
\end{equation}

\begin{lemma}
\label{lemma:2}
\begin{equation}
\label{eq:Delta}
\Delta(\Psi) = 2\lambda\iE[\bs_{\OO}/_{\displaystyle{\Psi}}]-\iE[\bs^2/_{\displaystyle{\Psi}}]
\end{equation}
\end{lemma}

\begin{proof}
From (\ref{eq:delta})
\begin{align*}
\iE[\Delta] &= \lambda^2-\iE[(\lambda-\bs_{\OO})^2+\bs_{\perp}^2] \\
&= \lambda^2-(\lambda^2 -2\lambda\iE[\bs_{\OO}]+\iE[\bs_{\OO}^2]) - \iE[\bs_{\perp}^2]  \\
&= 2\lambda\iE[\bs_{\OO}]-\iE[\bs^2]
\end{align*}
\widowpenalty=10000
\end{proof}

From Lemma~\ref{lemma:2} and (\ref{eq:gradient_full})\footnote{Below the short-hand notation $\iE[(\pm{\Psi})]$ implies that the 2 cases of $y(\bx,\Psi)=\bx^t\bar\bw\pm\Psi$ should be considered, with equal probability $\frac{1}{2}$ by assumption.}
\begin{align}
 \frac{1}{4}\Delta(\Psi) &= \eta\iE[r^2\lambda^2\cos^2\vartheta]-\eta^2\iE[r^4\lambda^2\cos^2\vartheta] \nonumber \\
\label{eq:e-delta}
-&\eta^2\Psi^2\iE[r^2] \\
-& \eta\iE[(\pm{\Psi})r\lambda\cos\vartheta] - 2\eta^2\iE[(\pm{\Psi})r^3\lambda\cos\vartheta] \nonumber
\end{align}

\begin{lemma}
\label{lemma:3}
\begin{equation*}
\iE[(\pm{\Psi})r\lambda\cos\vartheta] =\iE[(\pm{\Psi})r^3\lambda\cos\vartheta] = 0
\end{equation*}
\end{lemma}

\begin{proof}
The lemma follows from the assumed symmetry of ${\cal D}(\iX)$ with respect to the sign of $y_i-\bx_i^t\bar\bw$.
\end{proof}

It follows from Lemma~\ref{lemma:3} that
\begin{equation}
\label{cor:3}
\begin{split}
\frac{1}{4}\Delta(\Psi) =& \eta\iE[r^2\lambda^2\cos^2\vartheta]-\eta^2\iE[r^4\lambda^2\cos^2\vartheta] \\
&-\eta^2\Psi^2\iE[r^2] 
\end{split}
\end{equation}

We can now state the main theorem of this section.
\begin{theorem}
\label{theorem:1}
At time $t$ the expected convergence rate for training point $\bx$ is monotonically decreasing with the Difficulty Score $\Psi$ of $\bx$. If the step size coefficient is sufficiently small so that $\eta\le\frac{\iE[r^2\cos^2\vartheta]}{\iE[r^4\cos^2\vartheta]}$, it is likewise monotonically increasing with the distance $\lambda$ between the current estimate of the hypothesis $\bw_t$ and the optimal hypothesis $\bar\bw$.
\end{theorem}

\begin{proof}
From (\ref{cor:3})
\begin{equation*}
\frac{\partial\Delta(\Psi)}{\partial\Psi} = -8\eta^2\iE[r^2]\Psi \le 0
\end{equation*}
which proves the first statement. In addition,
\begin{equation*}
\frac{\partial\Delta(\Psi)}{\partial\lambda} = 8\eta\lambda \left ( \iE[r^2\cos^2\vartheta]-\eta\iE[r^4\cos^2\vartheta] \right )
\end{equation*}
If $\eta\le\frac{\iE[r^2\cos^2\vartheta]}{\iE[r^4\cos^2\vartheta]}$ then $\frac{\partial\Delta(\Psi)}{\partial\lambda}\ge 0$, and the second statement follows.
\widowpenalty=10000
\end{proof}

\begin{corollary}
\label{cor:1}
Although $\iE[\Delta(\Psi)]$ may be negative, $\bw_t$ always converges faster to $\bar\bw$ when the training points are sampled from easier examples with smaller $\Psi$. 
\end{corollary}

\begin{corollary}
\label{cor:2}
If the step size coefficient $\eta$ is small enough so that $\eta\le\frac{\iE[r^2\cos^2\vartheta]}{\iE[r^4\cos^2\vartheta]}$, we should expect faster convergence at the beginning of SCL. 
\end{corollary}

\subsubsection*{Convergence rate increases with current loss}

The main theorem in this sub-section states that for a fixed difficulty score $\Psi$, when the gradient step is small enough, convergence is monotonically \emph{increasing} with the loss of the point with respect to the current hypothesis. \emph{This is not true in general.} The second theorem in this section shows that when the difficulty score is not fixed, there exist hypotheses $\bw\in\cH$ for which the convergence rate is decreasing with the current loss.

\comment{this is never used, remove...
For any function of the data $F(\bX)=F(\bx,y)$, let $F_\Psi^+(\bx)=F(\bx,\bx^t\bar\bw+{\Psi})$ and $F_\Psi^-(\bx)=F(\bx,\bx^t\bar\bw-{\Psi})$. When $F(\bx,y)$ does not depend on the polar angles $\Phi$, using the assumed symmetry of ${\cal D}(\iX)$
\begin{equation}
\label{eq:expectation}
\begin{split}
\hspace{-3mm} \iE[F(\bx,y)/_{\displaystyle{\Psi}}] = \frac{1}{2}\int_0^\infty &\int_0^{2\pi} [F_\Psi^+(r,\vartheta) + F_\Psi^-(r,\vartheta)] \\
&r^d\sin^{d-1}\vartheta ~g(r,\vartheta)  d\vartheta dr  
\end{split}
\end{equation}
where 
\begin{equation*}
\begin{split}
\hspace{-3mm} g(r,\vartheta)=&\int_0^\pi\ldots\int_0^{2\pi} f(r,\vartheta,\Phi)  J(\Phi) d\varphi_1\ldots d\varphi_{d-1}\\
& J(\Phi)=\sin^{d-2}\varphi_1\ldots  \sin\varphi_{d-2} 
\end{split}
\end{equation*}
the marginal distribution of $[\vartheta,r]$.
}


Let $\Upsilon^2=L(\bX_i,\bw_t)$ denote the loss of $\bX_i$ with respect to the current hypothesis $\bw_t$. Define the angle $\beta\in [0,\frac{\pi}{2})$ as follows (see Fig.~\ref{fig:obtuse})
\begin{equation}
\label{eq:beta}
\beta=\beta(r,\Psi,\lambda) = \arccos (\min(\frac{ {\Psi}}{\lambda r},1))
\end{equation}

\begin{lemma}
\label{lemma:4}
The relation between $\Upsilon, \Psi, r, \vartheta$ can be written separately in 4 regions as follows (see Fig.~\ref{fig:obtuse}):
\begin{enumerate}[{A}1]
\item
$0\le\vartheta\le\pi-\beta,~y_i=\bx_i^t\bar\bw+{\Psi}\implies y_i=\bx_i^t\bw_t+\Upsilon, \\ \lambda r \cos\vartheta=\bx_i^t(\bar\bw-\bw_t)=-\Psi+\Upsilon$
\item
$\pi-\beta\le\vartheta\le\pi,~y_i=\bx_i^t\bar\bw+{\Psi}\implies y_i=\bx_i^t\bw_t-{\Upsilon},\\ \lambda r \cos\vartheta=-{\Psi}-{\Upsilon}$
\item
$0\le\vartheta\le\beta,~y_i=\bx_i^t\bar\bw-{\Psi}\implies y_i=\bx_i^t\bw_t+{\Upsilon},\\ \lambda r \cos\vartheta={\Psi}+{\Upsilon}$
\item
$\beta\le\vartheta\le\pi,~y_i=\bx_i^t\bar\bw-{\Psi}\implies y_i=\bx_i^t\bw_t-{\Upsilon},\\ \lambda r \cos\vartheta={\Psi}-{\Upsilon}$
\end{enumerate}
\end{lemma}
\vspace{-4mm}

\begin{proof}
We keep in mind that $\forall \bx_i$ and $\Psi$, there are 2 possible $y_i$ with equal probability. Recall that $\bar\bz$ denotes the projection of $\bar\bw$ on $\Omega_i$. In the planar section shown in Fig.~\ref{fig:obtuse},
\begin{description}
\item
$\bar\bz$ lies in the upper half space $\iff$ $y_i=\bx_i^t\bar\bw+{\Psi}$
\item
$\bar\bz$ lies in the lower half space  $\iff$ $y_i=\bx_i^t\bar\bw-{\Psi}$
\end{description}
This follows from 3 observations: $\bar\bx_i$ lies in the upper half space by the definition of the polar coordinate system, $\bx_i^t\bar\bw-y_i=\pm\Psi$, and
\begin{equation*}
0 = \bx_i^t\bar\bz - y_i= \bx_i^t (\bar\bz-\bar\bw)  +\bx_i^t\bar\bw-y_i
\end{equation*}

Next, let $\bz_t$ denote the projection of $\bw_t$ on $\Omega_i$. Then
\begin{equation*}
0 = \bx_i^t\bz_t - y_i= \bx_i^t (\bz_t-\bw_t)  +\bx_i^t\bw_t-y_i
\end{equation*}
When $\bar\bz$ lies in the upper half space, the following can be verified geometrically from Fig.~\ref{fig:obtuse}:
\begin{description}
\item
$0\le\vartheta\le\pi-\beta ~\Rightarrow~ \bx_i^t (\bz_t-\bw_t)\ge 0 ~\Rightarrow~ y_i=\bx_i^t\bw_t+{\Upsilon}$
\item
$\pi-\beta\le\vartheta\le\pi~\Rightarrow~\bx_i^t (\bz_t-\bw_t)\le 0 ~\Rightarrow~ y_i=\bx_i^t\bw_t-{\Upsilon}$
\end{description}
When $\bar\bz$ lies in the lower half space
\begin{description}
\item
$0\le\vartheta\le\beta \implies \bx_i^t (\bz_t-\bw_t)\ge 0 \implies y_i=\bx_i^t\bw_t+{\Upsilon}$
\item
$\beta\le\vartheta\le\pi\implies \bx_i^t (\bz_t-\bw_t)\le 0 \implies y_i=\bx_i^t\bw_t-{\Upsilon}$
\end{description}
\vspace{-3mm}
\widowpenalty=10000
\end{proof}


Next we analyze how the convergence rate at $\bx_i$ changes with $\Upsilon$. Let $\Delta(\Psi,\Upsilon)$ denote the expected convergence rate at time $t$, given fixed difficulty score $\Psi$ and fixed loss $\Upsilon$. From (\ref{cor:3}) $\Delta(\Psi,\Upsilon)=4\eta\iE[r^2\lambda^2\cos^2\vartheta/_{\displaystyle{\Upsilon}}]+O(\eta^2)$.

It is easier to analyze $\Delta(\Psi,\Upsilon)$ when using the Cartesian coordinates, rather than polar, in the $2D$ plane defined by the vectors $\iO=\bar\bw-\bw_t$ and $\bar\bz-\bar\bw$ (see Fig.~\ref{fig:obtuse}); thus we define $u = r \cos\vartheta, ~v= r \sin\vartheta$. The 4 cases listed in Lemma~\ref{lemma:4} can be readily transformed to this coordinate system as follows $\{0\le\vartheta\le\beta\}\Leftrightarrow \{\lambda u\ge\Psi\}$, $\{\beta\le\vartheta\le\pi-\beta\}\Leftrightarrow \{-\Psi\le \lambda u\le\Psi\}$, and $\{\pi-\beta\le\vartheta\le\pi\}\Leftrightarrow \{\lambda u\le -\Psi\}$:

\begin{enumerate}[{A}1]
\item
$\lambda u\ge -\Psi: ~~\lambda u=-{\Psi}+{\Upsilon}$
\item
$\lambda u\le -\Psi: ~~\lambda u=-{\Psi}-{\Upsilon}$
\item
$\lambda u\ge\Psi: ~~~~~\lambda u={\Psi}+{\Upsilon}$
\item
$\lambda u\le\Psi: ~~~~~\lambda u={\Psi}-{\Upsilon}$
\end{enumerate}

Define
\begin{equation*}
\nabla = \frac{f(\frac{\psi + \Upsilon}{\lambda}) -f(\frac{\psi - \Upsilon}{\lambda})-f(\frac{-\psi + \Upsilon}{\lambda})+f(\frac{-\psi - \Upsilon}{\lambda})}{f(\frac{\psi + \Upsilon}{\lambda}) +f(\frac{\psi - \Upsilon}{\lambda})+f(\frac{-\psi + \Upsilon}{\lambda})+f(\frac{-\psi - \Upsilon}{\lambda})}
\end{equation*}
Clearly $-1 \leq \nabla \leq 1$. 

\begin{theorem}
\label{theorem:2}
Assume that the gradient step size is small enough so that we can neglect second order terms $O(\eta^2)$, and that $\frac{\partial\nabla}{\partial\Upsilon} \geq \frac{\psi}{\Upsilon} - \frac{\Upsilon}{\psi}~\forall\Upsilon$. Fix the difficulty score at $\Psi$. At time $t$ the expected convergence rate is monotonically increasing with the loss $\Upsilon$ of the training point $\bx$.
\end{theorem}

\begin{proof}
In the coordinate system defined above $\Delta(\Psi,\Upsilon)=4\eta\iE[\lambda^2 u^2/_{\displaystyle{\Upsilon}}]+O(\eta^2)$. We compute $\Delta(\Psi,\Upsilon)$ separately in each region, marginalizing out $v$ based on the following
\begin{equation*}
\int  \int_0^{\infty} \lambda^2 u^2 v^{d-1}f(u,v) dv du  = \int \lambda^2 u^2 f(u)du
\end{equation*}
where $f(u)$ denotes the marginal distribution of $u$. 

Let $u_i$ denote the value of $u$ corresponding to loss $\Upsilon$ in each region A1-A4, and $\frac{1}{2}f(u_i)$ its density. $\Delta(\Psi,\Upsilon)$ takes 4 discrete values, one in each region, and its expected value is therefore $\Delta(\Psi,\Upsilon)=4\eta\sum_{i=1}^4 \lambda^2 u_i^2 \frac{f(u_i)}{\sum_{i=1}^4 f(u_i)}$. It can readily be shown that
\begin{equation}
\begin{split}
\frac{1}{4\eta}\Delta&(\psi,\Upsilon)= 
 \psi^2 + \Upsilon^2 +2\psi\Upsilon\nabla
\end{split}
\end{equation}
and subsequently 
\begin{equation}
\begin{split}
\frac{1}{4\eta}\frac{\partial\Delta(\psi,\Upsilon)}{\partial\Upsilon} &= 2\Upsilon+2\psi\Upsilon
~\frac{\partial\nabla}{\partial\Upsilon}+2\psi~\nabla \\
&\geq 2\Upsilon + 2\psi\Upsilon
~\frac{\partial\nabla}{\partial\Upsilon}-2\psi
\end{split}
\end{equation}

Using the assumption that $\frac{\partial\nabla}{\partial\Upsilon} \geq \frac{\psi}{\Upsilon} - \frac{\Upsilon}{\psi}~\forall\Upsilon$, we have that
\begin{equation*}
\frac{1}{8\eta}\frac{\partial\Delta(\psi,\Upsilon)}{\partial\Upsilon} \geq \Upsilon +\psi\Upsilon~\frac{\psi-\Upsilon}{\psi\Upsilon}-\psi=0
\end{equation*}
\widowpenalty=10000
\end{proof}

\begin{corollary}
\label{corol:3}
For any $c\in \mathbb{R}^+$, if $\nabla$ is $(c-\frac{1}{c})$-lipschitz then \\ $\frac{\partial\Delta(\psi,\Upsilon)}{
\partial\Upsilon} \geq 0$ for any $\Upsilon\geq c~\psi$.
\end{corollary}

\begin{corollary}
\label{corol:4}
If  ${\cal D}(\iX/_{\displaystyle{\Psi}})=k(\Psi)$ over a compact region and $\eta$ small enough, then $\frac{\partial\Delta(\psi,\Upsilon)}{\partial\Upsilon} \geq 0$ for all $\Upsilon$ excluding the boundaries of the compact region. If in addition $\Upsilon>\Psi$, then $\frac{\partial\Delta(\psi,\Upsilon)}{\partial\Upsilon} \geq 0$ almost surely.
\end{corollary}

\begin{theorem}
\label{theorem:3}
Assume that ${\cal D}(\iX)$ is continuous and that $\bar\bw$ is realizable. Then there are always hypotheses $\bw\in{\cH}$ for which the expected convergence rate under ${\cal D}(\iX)$ is monotonically decreasing with the loss $\Upsilon$ of the sampled points.
\end{theorem}

\begin{proof}
We shift to a hyperspherical coordinate system in $\iR^{d+1}$ similar as before, but now the pole (origin) is fixed at $\bw_t$. For the gradient step $\bs$, it can be shown that:
\begin{equation}
\begin{split}
\bs &= -\sgn{(\bx_i^t\bw_t-y_i)}2\eta\bx_i \Upsilon   \\
\bs_{\OO}&=\bs\cdot\frac{\bar\bw-\bw_t}{\lambda}=\pm \frac{2\eta}{\lambda}r\lambda\cos\vartheta~{\Upsilon}
\end{split}
\end{equation}

Let $\Delta(\Upsilon)$ denote the expected convergence rate at time $t$, given a fixed loss $\Upsilon$. From Lemma~\ref{lemma:2}
\begin{equation*}
\begin{split}
\Delta(\Upsilon) = 2\eta\Upsilon\bigg (&\iE[r\cos\vartheta/_{\displaystyle{\bx_i^t\bw_t-y_i=-\Upsilon}}] - 
\\ &\iE[r\cos\vartheta/_{\displaystyle{\bx_i^t\bw_t-y_i=\Upsilon}}]\bigg ) - \iE[(2\eta r\Upsilon)^2] \\
&\hspace{-15mm}\triangleq 2\eta\Upsilon Q(r,\vartheta,\bw_t) - 4\eta^2\Upsilon^2\iE[r^2]
\end{split}
\end{equation*}

If $\bw=\bar\bw$, then $Q(r,\vartheta,\bw)=0$ from the symmetry of ${\cal D}(\iX)$ with respect to $\Psi$. From the continuity of ${\cal D}(\iX)$, there exists $\delta>0$ such that if $\Vert \bw-\bar\bw\Vert_2<\delta$, then $\Vert Q(r,\vartheta,\bw)-Q(r,\vartheta,\bar\bw)\Vert_2<\eta\Upsilon\iE[r^2]$, which implies that $\Delta(\Upsilon)<-2\eta^2\Upsilon^2\iE[r^2]<0$.
\widowpenalty=10000
\end{proof}

\subsection{Deep learning: simulation results}
\label{sec:theory-simulations}

While the corollaries above apply to a rather simple situation, when using the \emph{Difficulty Score} to guide SGD while minimizing the convex regression loss, their predictions can be empirically tested with the deep learning architecture and loss which are described in Section~\ref{sec:deep-learning}. There an additional challenge is posed by the fact that the empirical ranking is not based on the ideal definition given in Def.~\ref{def:diff_score}, but rather on an estimate derived from another classifier. 

Still, the empirical results as shown in Fig.~\ref{fig:empirical-grad-dist} demonstrate agreement with the theoretical analysis of the linear regression loss. Specifically, in epoch 0 there is a big difference between the average errors in estimating the gradient direction, which is smallest for the easiest examples and highest for the most difficult examples as predicted by Corollary~\ref{cor:1}. This difference in significantly reduced after 10 epochs, and becomes insignificant after 20 epochs, in agreement with Corollary~\ref{cor:2}.

\vspace{-3mm}
\begin{figure}[th!]
	\centering
	\includegraphics[width=0.4\textwidth]{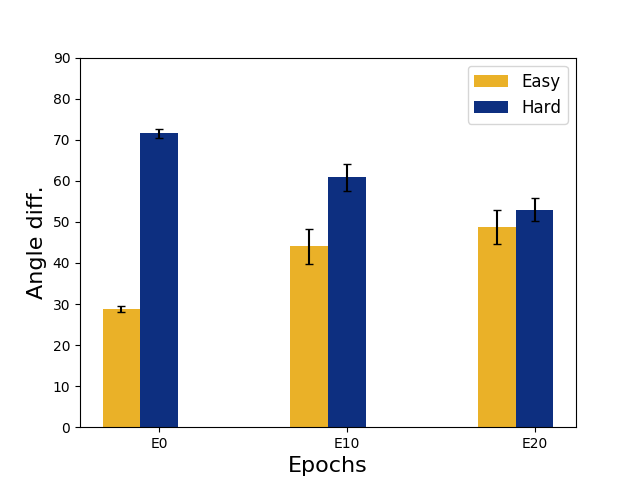}
    \caption{Results using the empirical setup described in Section~\ref{sec:deep-learning}. The average angular difference (in degrees) between the gradient computed based on a batch of 100 examples, and the true gradient based on all the training examples, is shown for 2 cases: the easiest examples (yellow) and the most difficult training examples (blue). Standard error bars are plotted based on 100 repetitions. Three Condition are shown: beginning of training (E0), 10 epochs into training (E10), and 20 epochs into training (E20).
    \label{fig:empirical-grad-dist}}
\end{figure}

\paragraph{Discussion.}

Fig.~\ref{fig:empirical-grad-dist} shows that the variance in the direction of the gradient step defined by easier points is significantly smaller than that defined by difficult points, especially at the beginning of training. This is advantageous when the initial point $\bw_0$ does not lie in the basin of attraction of the desired global minimum $\bar\bw$, and if, in agreement with Lemma~\ref{lemma:1}, the pronounced shared component of the easy gradient steps points in the direction of the global minimum, or a more favorable local minimum; then the likelihood of escaping the local minimum decreases with a point's \emph{Difficulty Score}. This scenario suggests another possible advantage for curriculum learning at the initial stages of training.

\section{Curriculum learning in deep networks}
\label{sec:deep-learning}

As discussed in the introduction, a practical curriculum learning method should address two main questions: how to rank the training examples, and how to modify the sampling procedure based on this ranking. Solutions to these issues are discussed in Section~\ref{sec:algorithm}. In Section~\ref{sec:empirical} we discuss the empirical evaluation of our method.

\subsection{Method}
\label{sec:algorithm}

\subsubsection*{Ranking examples by knowledge transfer}

The main novelty of our proposed method lies in this step, where we rank the training examples by estimated difficulty in the absence of human supervision. Difficulty is estimated based on knowledge transfer from another classifier. Here we investigate transfer from a more powerful learner.

It is a common practice now to treat one of the upstream layers of a pre-trained network as a representation (or embedding) layer. This layer activation is then used for representing similar objects and train a simpler classifier (such as SVM, or shallower NNs) to perform a different task, related but not identical to the original task the network had been trained on. In computer vision such embeddings are commonly obtained by training a deep network on the recognition of a very large database such as ImageNet \citep{imagenet_cvpr09}. These embeddings have been shown to provide better semantic representations of images (as compared to more traditional image features) in a number of related tasks, including the classification of small datasets \citep{sharif2014cnn}, image annotation \citep{donahue2015long} and structured predictions \citep{hu2016learning}. 

Following  this practice, the activation in the penultimate layer of a large and powerful pre-trained network is the loci of knowledge transfer from one network to another. Repeatedly, as in \citep{sharif2014cnn}, it has been shown that competitive performance can be obtained by training a shallow classifier on this representation in a new related task. Here we propose to use the confidence of such a classifier, e.g. the margin of an SVM classifier, as the estimator for the difficulty of each training example. This measure is then used to sort the training data. We note that unlike the traditional practice of reusing a pre-trained network, here we only transfer information from one learner to another. The goal is to achieve a smaller classifier that can conceivably be used with simpler hardware, without depending on access to the powerful learner at test time.


\subsubsection*{Scheduling the appearance of training examples}

In agreement with prior art, e.g. the definition of curriculum in \cite{bengio2009curriculum}, we investigate curriculum learning where the scheduling of examples changes with time, giving priority to easier examples at the beginning of training. We explored two variants of the basic scheduling idea:

\textbf{Fixed.}
The distribution used to sample examples from the training data is gradually changed in fixed steps. Initially all the weight is put on the easiest examples. In subsequent steps the weight of more difficult examples is gradually increased, until the final step in which the training data is sampled uniformly (or based on some prior distribution on the training set).

\textbf{Adaptive.}
Similar to the previous mode, but where the length of each step is not fixed, but is being determined adaptively based on the current loss of the training data.

\subsection{Empirical evaluation}
\label{sec:empirical}

\subsubsection*{Experimental setup}

\textbf{Datasets.}
For evaluation we used 2 data sets: CIFAR-100 \citep{krizhevsky2009learning} and STL-10 \citep{coates2010analysis}. In all cases, as is commonly done, the data was pre-processed using global contrast normalization; cropping and flipping were used for STL-10.

\textbf{Network architecture.}
We used convolutional Neural Networks (CNN) which excel at image classification tasks. Specifically, we used two architectures which are henceforth denoted \textit{Large} and \textit{Small}, in accordance with the number of parameters. The \textit{Large} network is comprised of four blocks, each with two convolutional layers, ELU activation, and max-pooling. This is followed by a fully connected layer, for a total of 1,208,101 parameters. The \textit{Small} network consists of only three hidden layers, for a total of 4,557 parameters. During training, we applied dropout and $l2$ regularization on the weights, and used either SGD or ADAM to optimize the cross-entropy loss. 

\textbf{Scheduling mechanisms: control.}
As described above, our method is based on a scheduling design which favors the presentation of easier examples at the beginning of training. In order to isolate the contribution of scheduling by increasing level of difficulty as against other spurious consequences of data scheduling, we compared performance with the following control conditions: \texttt{control-curriculum}, identical scheduling mechanism but where the underlying ranking of the training examples is random and unrelated to estimated difficulty; and \texttt{anti-curriculum}, identical scheduling mechanism but favoring the more difficult examples at the beginning of training.

\subsubsection*{Controlling for Task difficulty}
\label{sec:difficulty}

\begin{figure*}[ht!]
	\centering
	\includegraphics[width=0.345\textwidth]{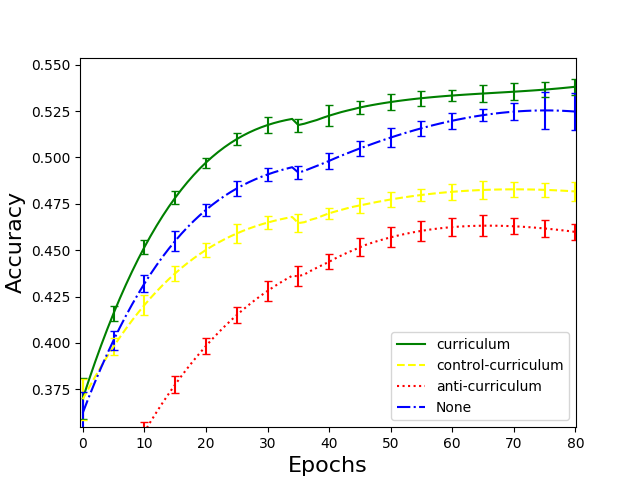}   
	\includegraphics[width=0.31\textwidth]{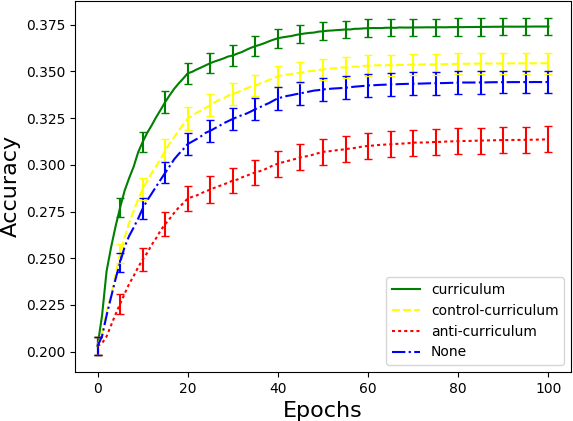}  \hspace{2mm}
	\includegraphics[width=0.31\textwidth]{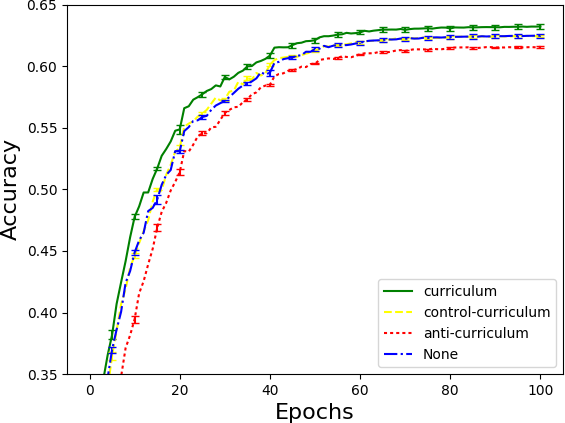}\\
    a) \hspace{0.32\textwidth} b) \hspace{0.32\textwidth} c)
    \caption{Accuracy in test classification as a function of training time. We show results with 4 scheduling methods: our method denoted \texttt{curriculum} (solid green), \texttt{control-curriculum} with random ordering (dashed yellow), \texttt{anti-curriculum} where more difficult examples are preferred at the beginning of training (dotted red), and \texttt{none} with no curriculum learning (dashdotted blue). a-b) Learning CIFAR100 task 1, where the \emph{Large} network is used in a) and the \emph{Small} network in b). c) Learning to classify STL-10 images.
    \label{fig:CIFAR100}}
\end{figure*}

Evidence from prior art is conflicting regarding where the benefits of curriculum learning lie, which is to be expected given the variability in the unknown sources of the curriculum supervision information and its quality. We observed in our empirical study that the benefits depended to a large extent on the difficulty of the task. We always saw faster learning at the beginning of the training process, while lower generalization error was seen only when the task was relatively difficult. We therefore employed controls for the following 3 sources of task difficulty:

\textbf{Inherent task difficulty.}
To investigate this factor, we take advantage of the fact that CIFAR-100 is a hierarchical dataset with 100 classes and 20 super-classes, each including 5 member classes. We therefore trained a network to discriminate the 5 member classes of 2 super-classes as 2 separate tasks: `small mammals' (task 1) and `aquatic mammals' (task 2). These are expected to be relatively hard learning tasks. We also trained a network to discriminate 5 random well separated classes: `camel', `clock', `bus', `dolphin' and `orchid' (task 3). This task is expected to be relatively easy.

\textbf{Size of classification network.} 
For a given task, classification performance is significantly affected by the size of the network and its architecture. We assume, of course, that we operate in the domain where the number of model parameters is smaller than can be justified by the training data (i.e., there is no overfit). We therefore used networks of different sizes in order to evaluate how curriculum learning is affected by \emph{task difficulty} as determined by the network's strength (see Fig.~\ref{fig:CIFAR100}a-b). In this comparative evaluation, the smaller the network is, the more difficult the task is likely to be (clearly, many other factors participate in the determination of task difficulty).

\textbf{Regularization and optimization.}
Regularization is used to constrain the family of hypotheses, or models, so that they possess such desirable properties as smoothness. Regularization effectively decreases the number of degrees of freedom in the model. In fact, most optimization methods, other then vanilla stochastic gradient descent, incorporate some form of regularization and smoothing, among other inherent properties. Therefore the selection of optimization method also plays a role in determining the effective size of the final network. 

\subsubsection*{Results}

Fig.~\ref{fig:CIFAR100}a shows typical results when training the  \textit{Large} CNN (see network's details above) to classify a subset of 5 CIFAR100 images (task 1 as defined above), using slow learning rate and Adam optimization. In this setup we see that curriculum learning speeds up the learning rate at the beginning of the training, but converges to the same performance as regular training. When we make learning more difficulty by using the  \textit{Small} network, performance naturally decreases, but now we see that curriculum learning also improves the final generalization performance (Fig.~\ref{fig:CIFAR100}b). Similar results are shown for the STL-10 dataset (Fig.~\ref{fig:CIFAR100}c).

\begin{figure}[th!]
	\centering
	\includegraphics[width=0.42\textwidth]{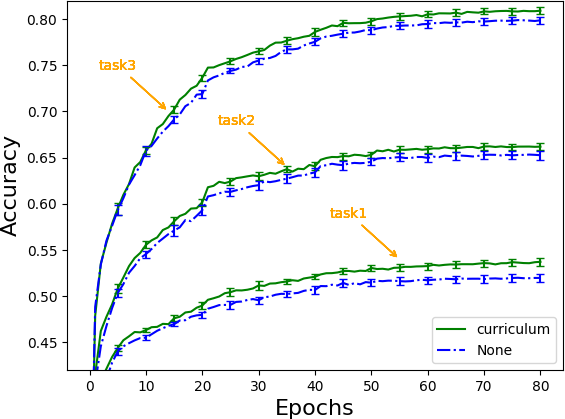}\\
    \caption{Accuracy in test classification in 3 tasks each involving 5 classes from the CIFAR100 datasets (see Section~\ref{sec:difficulty}). In each task, the performance of learning with a curriculum (solid green) and without (dashdotted blue) is shown.}
    \label{fig:CIFAR-subsets}
\end{figure}

\comment{ 
\begin{figure*}[t!]
	\centering
	\includegraphics[width=0.32\textwidth]{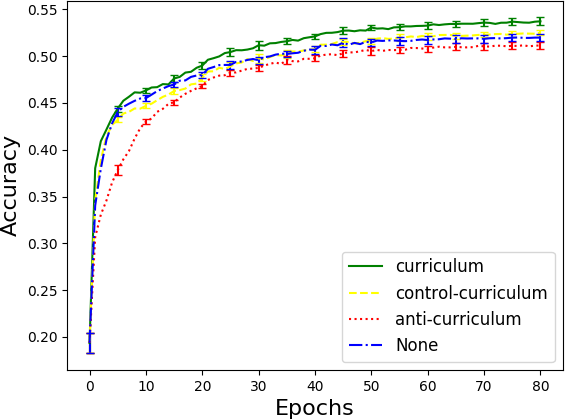}
	\includegraphics[width=0.32\textwidth]{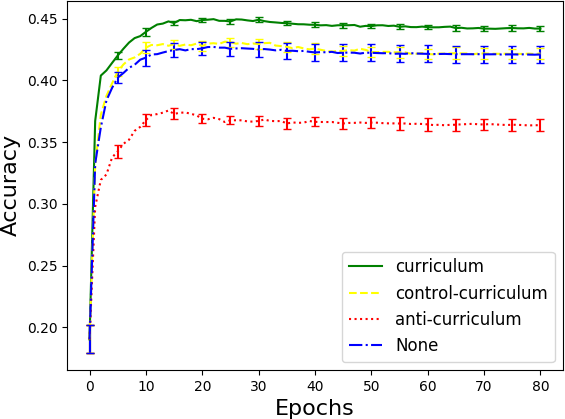}
	\includegraphics[width=0.32\textwidth]{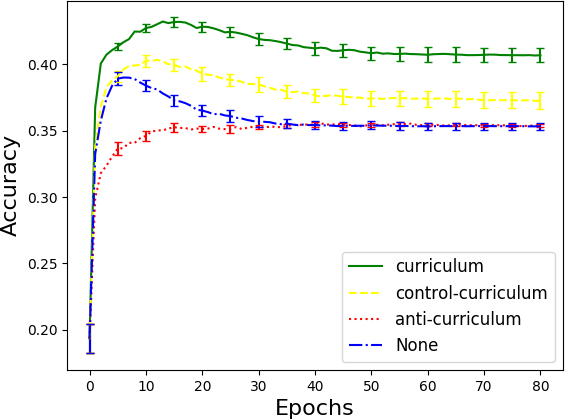}
    \\
    a) \hspace{0.32\textwidth} b) \hspace{0.32\textwidth} c)
    \caption{Accuracy in test classification as a function of training time, as seen in the learning of task 1. Panels are organized by level of regularization, from low (a) to high and most detrimental (c). 4 scheduling methods are shown as described in the caption of Fig.~\ref{fig:CIFAR100}.}
    \label{fig:CIFAR-subsets-reg}
\end{figure*}
}

Fig.~\ref{fig:CIFAR-subsets} shows comparative results when controlling for inherent task difficulty in the 3 tasks described above, using faster learning rate and SGD optimization. Task difficulty can be evaluated in retrospect from the final performance seen in each plot. As can be clearly seen in the figure, the improvement in final accuracy with curriculum learning is larger when the task is more difficult.
When manipulating the level of regularization, we see that while too much regularization always harms performance, curriculum learning is least affected by this degradation (results are omitted).

\section{Summary and Discussion}

We investigated curriculum learning, an extension  of stochastic gradient descent in which easy examples are more frequently sampled at the beginning of training. We started with the theoretical investigation of this strict definition in the context of linear regression, showing that curriculum learning accelerates the learning rate in agreement with prior empirical evidence. While not shedding light on its affect on the classifier's final performance, our analysis suggests that the direction of a gradient step based on "easy" examples may be more effective in traversing the input space towards the ideal minimum of the loss function. Specifically, we have empirically shown that the variance in the gradient direction of points increases with their difficulty when optimizing a non-convex loss function. Over-sampling the more coherent easier examples may therefore increase the likelihood to escape the basin of attraction of a low quality local minimum in favor of higher quality local minima even in the general non-convex case. 

We also showed theoretically that when the difficulty score of the training points is fixed, convergence is faster if the loss with respect to the current hypothesis is higher. This seems to be a very intuitive result, an intuition that underlies the boosting method for example. However, as intuitive as it might be, this is not always true when the prior data density is assumed to be continuous and when the optimal hypothesis is realizable. Thus the requirement that the difficulty score is fixed is necessary.

In the second part of this paper we described a curriculum learning method for deep networks. The method relies on knowledge transfer from other (pre-trained) networks in order to rank the training examples by difficulty. We described extensive experiments where we evaluated our proposed method under different task difficulty conditions and against a variety of control conditions. In all cases curriculum learning has been shown to increase the rate of convergence at the beginning of training, in agreement with the theoretical results. With more difficult tasks, curriculum learning improved generalization performance.

\section*{Acknowledgements}

This work was supported in part by a grant from the Israel Science Foundation (ISF) and by the Gatsby Charitable Foundations.

\bibliography{bib}

\begin{thebibliography}{20}
\providecommand{\natexlab}[1]{#1}
\providecommand{\url}[1]{\texttt{#1}}
\expandafter\ifx\csname urlstyle\endcsname\relax
  \providecommand{\doi}[1]{doi: #1}\else
  \providecommand{\doi}{doi: \begingroup \urlstyle{rm}\Url}\fi

\bibitem[Amodei et~al.(2016)Amodei, Ananthanarayanan, Anubhai, Bai, Battenberg,
  Case, Casper, Catanzaro, Cheng, Chen, et~al.]{amodei2016deep}
Amodei, D., Ananthanarayanan, S., Anubhai, R., Bai, J., Battenberg, E., Case,
  C., Casper, J., Catanzaro, B., Cheng, Q., Chen, G., et~al.
\newblock Deep speech 2: End-to-end speech recognition in english and mandarin.
\newblock In \emph{International Conference on Machine Learning}, pp.\
  173--182, 2016.

\bibitem[Bengio et~al.(2009)Bengio, Louradour, Collobert, and
  Weston]{bengio2009curriculum}
Bengio, Y., Louradour, J., Collobert, R., and Weston, J.
\newblock Curriculum learning.
\newblock In \emph{Proceedings of the 26th annual international conference on
  machine learning}, pp.\  41--48. ACM, 2009.

\bibitem[Chen et~al.(2015)Chen, Wilson, Tyree, Weinberger, and
  Chen]{chen2015compressing}
Chen, W., Wilson, J., Tyree, S., Weinberger, K., and Chen, Y.
\newblock Compressing neural networks with the hashing trick.
\newblock In \emph{International Conference on Machine Learning}, pp.\
  2285--2294, 2015.

\bibitem[Coates et~al.(2010)Coates, Lee, and Ng]{coates2010analysis}
Coates, A., Lee, H., and Ng, A.~Y.
\newblock An analysis of single-layer networks in unsupervised feature
  learning.
\newblock \emph{Ann Arbor}, 1001\penalty0 (48109):\penalty0 2, 2010.

\bibitem[Deng et~al.(2009)Deng, Dong, Socher, Li, Li, and
  Fei-Fei]{imagenet_cvpr09}
Deng, J., Dong, W., Socher, R., Li, L.-J., Li, K., and Fei-Fei, L.
\newblock {ImageNet: A Large-Scale Hierarchical Image Database}.
\newblock In \emph{CVPR09}, 2009.

\bibitem[Donahue et~al.(2015)Donahue, Anne~Hendricks, Guadarrama, Rohrbach,
  Venugopalan, Saenko, and Darrell]{donahue2015long}
Donahue, J., Anne~Hendricks, L., Guadarrama, S., Rohrbach, M., Venugopalan, S.,
  Saenko, K., and Darrell, T.
\newblock Long-term recurrent convolutional networks for visual recognition and
  description.
\newblock In \emph{Proceedings of the IEEE conference on computer vision and
  pattern recognition}, pp.\  2625--2634, 2015.

\bibitem[Graves et~al.(2017)Graves, Bellemare, Menick, Munos, and
  Kavukcuoglu]{graves2017automated}
Graves, A., Bellemare, M.~G., Menick, J., Munos, R., and Kavukcuoglu, K.
\newblock Automated curriculum learning for neural networks.
\newblock \emph{arXiv preprint arXiv:1704.03003}, 2017.

\bibitem[Hu et~al.(2016)Hu, Zhou, Deng, Liao, and Mori]{hu2016learning}
Hu, H., Zhou, G.-T., Deng, Z., Liao, Z., and Mori, G.
\newblock Learning structured inference neural networks with label relations.
\newblock In \emph{Proceedings of the IEEE Conference on Computer Vision and
  Pattern Recognition}, pp.\  2960--2968, 2016.

\bibitem[Jesson et~al.(2017)Jesson, Guizard, Ghalehjegh, Goblot, Soudan, and
  Chapados]{jesson2017cased}
Jesson, A., Guizard, N., Ghalehjegh, S.~H., Goblot, D., Soudan, F., and
  Chapados, N.
\newblock Cased: Curriculum adaptive sampling for extreme data imbalance.
\newblock In \emph{International Conference on Medical Image Computing and
  Computer-Assisted Intervention}, pp.\  639--646. Springer, 2017.

\bibitem[Kim et~al.(2015)Kim, Park, Yoo, Choi, Yang, and
  Shin]{kim2015compression}
Kim, Y.-D., Park, E., Yoo, S., Choi, T., Yang, L., and Shin, D.
\newblock Compression of deep convolutional neural networks for fast and low
  power mobile applications.
\newblock \emph{arXiv preprint arXiv:1511.06530}, 2015.

\bibitem[Krizhevsky \& Hinton(2009)Krizhevsky and
  Hinton]{krizhevsky2009learning}
Krizhevsky, A. and Hinton, G.
\newblock Learning multiple layers of features from tiny images.
\newblock 2009.

\bibitem[Krueger \& Dayan(2009)Krueger and Dayan]{krueger2009flexible}
Krueger, K.~A. and Dayan, P.
\newblock Flexible shaping: How learning in small steps helps.
\newblock \emph{Cognition}, 110\penalty0 (3):\penalty0 380--394, 2009.

\bibitem[Mitchell(1980)]{mitchell1980need}
Mitchell, T.~M.
\newblock \emph{The need for biases in learning generalizations}.
\newblock Department of Computer Science, Laboratory for Computer Science
  Research, Rutgers Univ. New Jersey, 1980.

\bibitem[Mitchell(2006)]{mitchell2006discipline}
Mitchell, T.~M.
\newblock \emph{The discipline of machine learning}, volume~9.
\newblock Carnegie Mellon University, School of Computer Science, Machine
  Learning Department, 2006.

\bibitem[Sharif~Razavian et~al.(2014)Sharif~Razavian, Azizpour, Sullivan, and
  Carlsson]{sharif2014cnn}
Sharif~Razavian, A., Azizpour, H., Sullivan, J., and Carlsson, S.
\newblock Cnn features off-the-shelf: an astounding baseline for recognition.
\newblock In \emph{Proceedings of the IEEE Conference on Computer Vision and
  Pattern Recognition Workshops}, pp.\  806--813, 2014.

\bibitem[Shrivastava et~al.(2016)Shrivastava, Gupta, and
  Girshick]{ShrivastavaGG16}
Shrivastava, A., Gupta, A., and Girshick, R.~B.
\newblock Training region-based object detectors with online hard example
  mining.
\newblock \emph{CoRR}, abs/1604.03540, 2016.
\newblock URL \url{http://arxiv.org/abs/1604.03540}.

\bibitem[Szegedy et~al.(2013)Szegedy, Zaremba, Sutskever, Bruna, Erhan,
  Goodfellow, and Fergus]{szegedy2013intriguing}
Szegedy, C., Zaremba, W., Sutskever, I., Bruna, J., Erhan, D., Goodfellow, I.,
  and Fergus, R.
\newblock Intriguing properties of neural networks.
\newblock \emph{arXiv preprint arXiv:1312.6199}, 2013.

\bibitem[Thrun \& Pratt(2012)Thrun and Pratt]{thrun2012learning}
Thrun, S. and Pratt, L.
\newblock \emph{Learning to learn}.
\newblock Springer Science \& Business Media, 2012.

\bibitem[Wang \& Cottrell(2015)Wang and Cottrell]{wang2015basic}
Wang, P. and Cottrell, G.~W.
\newblock Basic level categorization facilitates visual object recognition.
\newblock \emph{arXiv preprint arXiv:1511.04103}, 2015.

\bibitem[Zaremba \& Sutskever(2014)Zaremba and Sutskever]{zaremba2014learning}
Zaremba, W. and Sutskever, I.
\newblock Learning to execute.
\newblock \emph{arXiv preprint arXiv:1410.4615}, 2014.

\end{thebibliography}
\bibliographystyle{icml2018}

\comment{not needed at the moment
\appendix

\section*{Appendix}

\section{Details for omitted proofs}
\label{app:proofs}

Lemma~\ref{lemma:2}:
\begin{equation*}
\Delta(\Psi) = 2\lambda\iE[\bs_{\OO}/_{\displaystyle{\Psi}}]-\iE[\bs^2/_{\displaystyle{\Psi}}]
\end{equation*}

\begin{proof}
From (\ref{eq:delta})
\begin{align*}
\iE[\Delta] &= \lambda^2-\iE[(\lambda-\bs_{\OO})^2+\bs_{\perp}^2] \\
&= \lambda^2-(\lambda^2 -2\lambda\iE[\bs_{\OO}]+\iE[\bs_{\OO}^2]) - \iE[\bs_{\perp}^2]  \\
&= 2\lambda\iE[\bs_{\OO}]-\iE[\bs^2]
\end{align*}
\widowpenalty=10000
\end{proof}

Lemma~\ref{lemma:4}. The relation between $\Upsilon, \Psi, r, \vartheta$ can be written separately in 4 regions as follows (see Fig.~\ref{fig:obtuse}):
\begin{enumerate}[{A}1]
\item
$0\le\vartheta\le\pi-\beta,~y_i=\bx_i^t\bar\bw+{\Psi},~y_i=\bx_i^t\bw_t+{\Upsilon}\implies \lambda r \cos\vartheta=\bx_i^t(\bar\bw-\bw_t)=-\Psi+\Upsilon$
\item
$\pi-\beta\le\vartheta\le\pi,~y_i=\bx_i^t\bar\bw+{\Psi},~y_i=\bx_i^t\bw_t-{\Upsilon}\implies \lambda r \cos\vartheta=-{\Psi}-{\Upsilon}$
\item
$0\le\vartheta\le\beta,~y_i=\bx_i^t\bar\bw-{\Psi},~y_i=\bx_i^t\bw_t+{\Upsilon}\implies \lambda r \cos\vartheta={\Psi}+{\Upsilon}$
\item
$\beta\le\vartheta\le\pi,~y_i=\bx_i^t\bar\bw-{\Psi},~y_i=\bx_i^t\bw_t-{\Upsilon}\implies \lambda r \cos\vartheta={\Psi}-{\Upsilon}$
\end{enumerate}

\begin{proof}
We keep in mind that $\forall \bx_i$ and $\Psi$, there are 2 possible $y_i$ with equal probability. Recall that $\bar\bz$ denotes the projection of $\bar\bw$ on $\Omega_i$. In the planar section shown in Fig.~\ref{fig:obtuse},
\begin{description}
\item
$\bar\bz$ lies in the upper half space $\iff$ $y_i=\bx_i^t\bar\bw+{\Psi}$
\item
$\bar\bz$ lies in the lower half space  $\iff$ $y_i=\bx_i^t\bar\bw-{\Psi}$
\end{description}
This follows from 3 observations: $\bar\bx_i$ lies in the upper half space by the definition of the polar coordinate system, $\bx_i^t\bar\bw-y_i=\pm\Psi$, and
\begin{equation*}
0 = \bx_i^t\bar\bz - y_i= \bx_i^t (\bar\bz-\bar\bw)  +\bx_i^t\bar\bw-y_i
\end{equation*}

Next, let $\bz_t$ denote the projection of $\bw_t$ on $\Omega_i$. Then
\begin{equation*}
0 = \bx_i^t\bz_t - y_i= \bx_i^t (\bz_t-\bw_t)  +\bx_i^t\bw_t-y_i
\end{equation*}
When $\bar\bz$ lies in the upper half space, the following can be verified geometrically from Fig.~\ref{fig:obtuse}:
\begin{description}
\item
$0\le\vartheta\le\pi-\beta ~\Rightarrow~ \bx_i^t (\bz_t-\bw_t)\ge 0 ~\Rightarrow~ y_i=\bx_i^t\bw_t+{\Upsilon}$
\item
$\pi-\beta\le\vartheta\le\pi~\Rightarrow~\bx_i^t (\bz_t-\bw_t)\le 0 ~\Rightarrow~ y_i=\bx_i^t\bw_t-{\Upsilon}$
\end{description}
When $\bar\bz$ lies in the lower half space
\begin{description}
\item
$0\le\vartheta\le\beta \implies \bx_i^t (\bz_t-\bw_t)\ge 0 \implies y_i=\bx_i^t\bw_t+{\Upsilon}$
\item
$\beta\le\vartheta\le\pi\implies \bx_i^t (\bz_t-\bw_t)\le 0 \implies y_i=\bx_i^t\bw_t-{\Upsilon}$
\end{description}

\widowpenalty=10000
\end{proof}
}

\end{document}